\documentclass{article}

\usepackage[preprint]{neurips_2026}
\usepackage{booktabs}
\usepackage{multirow}

\usepackage[utf8]{inputenc} % allow utf-8 input
\usepackage[T1]{fontenc}    % use 8-bit T1 fonts
\usepackage{hyperref}       % hyperlinks
\usepackage{url}            % simple URL typesetting
\usepackage{booktabs}       % professional-quality tables
\usepackage{amsfonts}       % blackboard math symbols
\usepackage{nicefrac}       % compact symbols for 1/2, etc.
\usepackage{microtype}      % microtypography
\usepackage{xcolor}         % colors
\usepackage{graphicx}
\usepackage{subcaption}
\usepackage{algorithm}
\usepackage{algpseudocode}
\usepackage{amsmath}
\usepackage{amsthm}

\newtheorem{theorem}{Theorem}
\title{OS-Pruner: Pruning Chains-of-Thought of
Reasoning Models via Optimal Stopping}

\author{%
  Mohammed Ehab \\
  MIT\\
  \texttt{ehab02@mit.edu} \\
  \And Aymane El Gadarri \\
  MIT\\
  \texttt{aymelgad@mit.edu} \\
  \And Vivek Farias \\
  MIT\\
  \texttt{vivekf@mit.edu} \\ 
  \AND Adam Jozefiak \\
  MIT\\
  \texttt{jozefiak@mit.edu} \\
  \And Ciamac C. Moallemi \\
  Columbia Business School \\
  \texttt{ciamac@gsb.columbia.edu}
}

\begin{document}

\maketitle

\begin{abstract}

  Large Language Models (LLMs) have achieved remarkable success in complex reasoning tasks through Chain-of-Thought (CoT) prompting. However, these models often exhibit "computational overthinking," generating redundant reasoning steps that increase latency and cost without improving accuracy. Recent studies suggest that CoT trajectories can be significantly pruned, yet existing methods often rely on forcing a static thinking budget, heuristic filtering, sub-optimal early exit via classification, or expensive re-training. In this paper, we introduce OS-Pruner, a lightweight plug-in framework that formulates chain-of-thought pruning as an optimal stopping problem. Given a reasoning prefix, OS-Pruner learns whether further reasoning is worth its token cost by optimizing an explicit utility that trades off final-answer accuracy against generated length. Our novel formulation enables the model to dynamically assess the sufficient point of termination for a reasoning chain. OS-Pruner is designed to be lightweight during both training and inference, and to provide users with fine-grained control over the reasoning-effort vs. accuracy trade-off. On diverse reasoning benchmarks and base models, OS-Pruner achieves 20-60\% reduction in generation length with minimal accuracy sacrifice.
\end{abstract}

\section{Introduction}
Chain-of-thought (CoT) reasoning has become a central mechanism for improving the performance of large language models on difficult mathematical, scientific, and symbolic tasks. By generating intermediate reasoning steps before producing a final answer, models can decompose a problem, explore candidate solutions, and recover from local mistakes. This additional test-time computation has proved especially important for recent reasoning models, where longer internal traces often correlate with stronger problem-solving ability \cite{weireasoning, wang2023selfconsistency,deepseek-ai_deepseek-r1_2025}. 

However, longer reasoning is not always better. In practice, reasoning models frequently continue generating after the right solution would have been decoded. They restate previous arguments, verify already-settled conclusions, pursue low-value side calculations, or simply delay producing the final answer. This \emph{overthinking} behavior increases latency and inference cost without improving accuracy, and sometimes even degrades it. More importantly, the amount of useful reasoning varies sharply across examples: a simple problem may require only a few steps, while a hard problem may benefit from a long search.

Recent work has begun to address this inefficiency from several angles. Budget-control methods, such as budget forcing and budget guidance, steer generation toward a target amount of thinking, but the target budget is typically imposed externally rather than chosen from the evolving state of the reasoning trace \cite{muennighoff_s1_2025,li_steering_2025}. Model-side compression methods, including length-penalized fine-tuning, token skipping, step-entropy compression, and related reinforcement-learning objectives, train models to produce shorter reasoning traces overall \cite{luo_o1-pruner_2025, xia_tokenskip_2025, li_making_2026, li_drpo_2026}. These approaches can be effective, but they require a costly retraining of the base reasoning model. A third line of work studies early stopping directly, using signals such as answer entropy, answer convergence, hidden-state probes, or verifier models to decide whether the current prefix is already sufficient \cite{laaouach_halt-cot_2025, liu_answer_2025, zhang2025selfverify, jiang_flashthink_2025,mao2025ES}. These methods show that substantial redundancy remains in many CoT traces. Yet most of them reduce stopping to classification or thresholding: a model predicts whether the current state is ``confident enough,'' and generation stops once a manually chosen threshold is crossed.

We argue that CoT pruning is more naturally a sequential decision problem. After each reasoning step, the system must choose between two actions: stop now and produce the final answer, or continue reasoning and pay for more tokens in the hope of improving the answer. This leads to an optimal stopping formulation: a CoT prefix should be extended only when the expected improvement from further reasoning outweighs its additional cost, and terminated once that tradeoff is no longer favorable.

We introduce \textbf{OS-Pruner}, a lightweight plug-in framework for stopping chains of thought via optimal stopping. Given an input problem and a partially generated CoT, OS-Pruner learns a stopping policy that estimates whether the model should continue reasoning or terminate and produce the final answer. The policy is trained to maximize an explicit utility: answer accuracy minus a length penalty. A single scalar parameter controls this penalty, allowing users to choose the desired operating point on the accuracy-efficiency frontier.

OS-Pruner is designed to be efficient in both training and inference. During preprocessing, we generate reasoning traces with a frozen base model and evaluate what would happen if the trace were stopped after each reasoning step. This produces fine-grained supervision for every prefix by looking at whether it is already sufficient to answer correctly. Because the base model is fixed, these rewards can be precomputed and reused, so training the stopping policy does not require repeated on-policy rollouts. Architecturally, OS-Pruner adds only a small policy head, with light fine-tuning of a small number of final transformer layers. At inference time, the policy is invoked only at reasoning-step boundaries, making inference highly efficient on modern LLM serving systems.

We summarize our contributions as follows:

\begin{itemize}
    \item We provide a rigorous formulation for CoT pruning as an optimal stopping problem, with an explicit reward that trades off final answer accuracy against generated length.
    \item We propose OS-Pruner, a lightweight plug-in stopping policy that operates on reasoning prefixes without modifying the frozen base reasoning model.
    \item We show on various mathematical reasoning benchmarks that OS-Pruner substantially reduces reasoning length while preserving accuracy, achieving 20-60\% generation-length reduction at low accuracy cost across diverse reasoning models.
    \item We show that even models trained explicitly for brevity still suffer from overthinking, with OS-Pruner delivering a performance gain when appended to a strong model-side method.
    % \item We show that OS-Pruner still improves the accuracy-efficiency frontier \aymane{change wording.} when plugged on top of models that are specifically retrained for brevity.
\end{itemize}

\section{Related Work}

\paragraph{Chain of Thought Reasoning and Test-time compute.}
Chain-of-thought (CoT) prompting has become a standard mechanism for improving the reasoning ability of large language models by externalizing intermediate computations before producing a final answer~\citep{weireasoning}. Follow-up work has shown that sampling multiple reasoning traces and aggregating their answers can further improve performance~\citep{wang2023selfconsistency}. More recently, reasoning-specialized models such as DeepSeek-R1 have pushed this paradigm further by scaling test-time computation through long, explicit reasoning traces~\citep{Guo2025R1}. These advances suggest that additional reasoning tokens can be valuable, especially on difficult mathematical and scientific problems. However, they also introduce a new efficiency problem: models often continue generating after the answer is already determined, producing redundant verification, repeated derivations, or low-value exploratory steps. Our work addresses this inefficiency by learning when a partially generated reasoning trace is already sufficient for final-answer generation.

\paragraph{Model-side methods for efficient reasoning.}
One line of work improves reasoning efficiency by modifying the base model so that it produces shorter CoTs. O1-Pruner introduces length-harmonizing fine-tuning to reduce the CoT overhead while preserving accuracy~\citep{luo_o1-pruner_2025}, while DRPO decouples length-based rewards for correct and incorrect rollouts to avoid penalizing valid but initially long reasoning trajectories~\citep{li_drpo_2026}. Other methods construct compressed reasoning traces for supervised fine-tuning. TokenSkip removes less important tokens from CoT trajectories and trains models to skip redundant parts of the reasoning process~\citep{xia_tokenskip_2025}. Step Entropy estimates the informational contribution of reasoning steps and compresses low-value steps with a special skip mechanism~\citep{li_making_2026}. A related direction avoids explicit natural-language reasoning altogether by moving some computation into continuous latent states~\citep{cheng_compressed_2024,hao_training_2025}. These model-side approaches can substantially reduce average generation length, but they require costly retraining or fine-tuning the generator. %and often entangle several behaviors: writing more concisely, changing the search strategy, avoiding dead ends, and deciding when to terminate.
On top, they can introduce unexpected behavior and degraded performance for non-reasoning and out-of-distribution tasks (e.g. commonsense questions) as they alter the base model and their generalization behavior is not well-understood.

\paragraph{Budgeted and prompt-side control.}
Another family of methods controls reasoning length at inference time through explicit budgets or prompting. Budget forcing, introduced in s1, controls test-time compute by forcing the model either to continue thinking, for example by appending continuation cues, or to terminate once a prescribed budget is reached~\citep{muennighoff_s1_2025}. Budget Guidance instead trains a lightweight predictor of remaining reasoning length and uses it to steer next-token generation toward a target budget without fine-tuning the base LLM~\citep{li_steering_2025}. Prompt-based methods such as Chain of Draft encourage models to write shorter intermediate traces by changing the format of reasoning itself~\citep{xu_chain_2025}. These approaches make reasoning length more controllable, but the target amount of thinking is typically imposed externally or guided toward a pre-specified budget. OS-Pruner instead chooses the stopping time from the evolving state of the reasoning trace: after each reasoning step, it asks whether continuing is expected to improve the answer enough to justify the additional token cost. This allows reasoning length to vary naturally with both problem difficulty and the information already present in the prefix.

\paragraph{Early stopping and self-verification.}
The closest line of work studies whether CoT generation can be stopped before the model emits its full reasoning trace. Training-free methods use signals derived from the model's current behavior. HALT-CoT stops when the entropy of the answer distribution falls below a threshold~\citep{laaouach_halt-cot_2025}; Think Just Enough uses sequence-level entropy as a confidence signal for reasoning models~\citep{sharma2025thinkjustenough}. Dynamic Early Exit monitors confidence at reasoning transition points to decide whether to truncate the next segment~\citep{yang2025dynamicearlyexitreasoning}. 
% Answer-convergence methods observe that model answers often stabilize before the end of the CoT and stop once consecutive reasoning steps produce consistent answers~\citep{liu_answer_2025,mao2025ES} 
REFRAIN combines reflective-redundancy detection with adaptive thresholding to stop once additional reasoning appears redundant~\citep{sun2025stopenough}. These methods demonstrate that substantial redundancy remains in many CoT traces, but they generally rely on heuristic or calibrated thresholds over confidence, entropy, convergence, or redundancy scores.

Training-based early-exit methods learn a verifier or probe that predicts whether the current reasoning prefix is sufficient. To the best of our knowledge, all such methods in the literature rely on training a classifier to predict correctness or a proxy thereof, and stop once the classifier confidence exceeds a fixed threshold. Hidden-state probing work shows that reasoning models contain internal signals correlated with answer correctness and train a classifier based on them~\citep{zhang2025classifier}. Answer Convergence trains a classifier to predict whether the answer remains stable until the end of the generation~\citep{liu_answer_2025}. FlashThink finetunes an external LLM prompted to verify the reasoning completeness to serve as a stopping classifier~\citep{jiang_flashthink_2025}. % These approaches are highly relevant to OS-Pruner because they also operate at intermediate reasoning states. However, most existing learned early-stopping methods reduce the decision to classification followed by thresholding: a partial CoT is judged correct, stable, or sufficient, and generation stops when the score exceeds a fixed threshold.

\section{Method}\label{sec:method}

We propose \textbf{OS-Pruner}, a lightweight, training-based plug-in method for early stopping. OS-Pruner is motivated by the observation that while multiple existing methods approach early stopping by training a classifier to decide when to stop, based on a fixed confidence threshold, the problem is more naturally and optimally framed as an optimal stopping problem. Instead of a classifier, a policy can be trained via reinforcement learning to produce probabilities of continuing only when the expected future improvement in answer quality exceeds the additional token cost.

\subsection{Problem Formulation}
Let $x$ denote an input problem. A reasoning LLM autoregressively generates a CoT $y = (y_1, \dots,y_T)$ as a sequence of reasoning steps (e.g. each $y_i$ corresponds to a new paragraph). Let $y_{\leq i}$ be the prefix consisting of the first $i$ steps. After observing $y_{\leq i}$, a stopping policy may terminate generation of the CoT and force the base model to output the final answer conditioned on this prefix.

OS-Pruner learns a parametric stopping policy $\pi_{\theta}$ which is a mapping of the current prefix to a probability of stopping: $\pi_\theta(y_{\le i}| x) \in [0,1]$.
Under this parametrization, the distribution of the stopping time $\tau$ is given by:

\begin{equation}
    \mathbb P(\tau = i) = [\prod\limits_{j<i} 1- \pi_{\theta}(y_{\le j}|x)] \cdot \pi_{\theta}(y_{\le i}|x)
\end{equation}

To train our policy, we design a reward function for stopping at the $i^{th}$ step. We denote by $A(y_{\le i}|x)$ the accuracy of the resulting answer (e.g. for numerical problems, 1 if the answer is correct and 0 if not.) And we denote by $L(y_{\le i})$ the number of tokens generated up to step $i$. We define that reward as:

\begin{equation}
    r(y_{\le i}|x)=A(y_{\le i}|x) - \lambda L(y_{\le i}).
\end{equation}

Conditionally on a problem-CoT pair $(x,y)$ with $T$ reasoning steps, the expected reward for stopping can be expressed as:
\begin{equation}\label{eq:expected-reward}
    R_\theta(y|x) = \sum\limits_{i=1}^T [\prod\limits_{j<i} 1-\pi_{\theta}(y_{\le j}|x)] \cdot \pi_{\theta}(y_{\le i}|x) \cdot r(y_{\le i}|x),
\end{equation}

Given a distribution of questions $\mathcal D_x$ and a fixed LLM policy $\pi^{LLM}(\cdot | x)$, an optimal stopping policy maximizes the objective:
\begin{equation}\label{eq:OS-objective}
    \mathcal J(\theta) = \mathbb E_{x \in \mathcal D_x}[\mathbb E_{y \sim \pi^{LLM}(\cdot | x)}[R_\theta(x,y)] ].
\end{equation}

\subsection{Interpretation of $\lambda$}

From an operations perspective, $\lambda$ can be interpreted as the increase in confidence in accuracy needed to justify paying for one more token. From an optimization perspective, by sweeping over a wide range of values of $\lambda$, we trace a frontier of policies trading accuracy for length reduction. The parameter $\lambda$ allows for fine-grained control over that trade-off.

\section{Value Function View and Classification Sub-Optimality}

The value function view sheds light on the advantage of our optimal stopping formulation. The key distinction between OS-Pruner and confidence-threshold methods is that stopping is not determined solely by whether the current answer is likely to be correct. The policy must compare the value of stopping now with the value of continuing, paying the cost of more tokens to increase confidence in producing a correct answer.
Let $V_\lambda(x, y_{\leq i})$ denote the optimal expected reward from state $(x, y_{\leq i})$. It satisfies the Bellman equation

\begin{equation}
V_\lambda(x, y_{\leq i}) = \max \{A(y_{\leq i} \mid x),\mathbb E_{y_{i+1} \sim \pi^{LLM}(. \mid x, y_{\leq i})}[V_{\lambda}(y_{\le i+1})-\lambda(L(y_{\le i+1})-L(y_{\le i})))\}.
\end{equation}
The first term is the value of stopping immediately. The second term is the value of generating an extra reasoning step and then acting optimally, which we will call the continuation value $C_\lambda(x, y_{\leq i})$.

Thus, the optimal stopping policy stops exactly when 

\begin{equation}
    A(y_{\leq i} \mid x) \geq C_\lambda(x, y_{\leq i}).
\end{equation}

In contrast, classification-based correctness methods stop according to the condition

\begin{equation}
    A(y_{\leq i} \mid x) \geq \gamma.
\end{equation}

As opposed to the fixed threshold $\gamma$, continuation value is state dependent. Two partial CoTs with the same correctness can require different decisions if one has promising future reasoning paths and the other does not.
% \subsection{Learning a stopping policy from fixed traces}

% Equation \ref{eq:expected-reward} ensures that this objective is differentiable in $\theta$ and since all CoTs and rewards are precomputed from the base model, OS-Pruner does not require expensive on-policy rollouts while training the stopping policy.
%\subsection{Inference}
%At inference time, we run the base model's generation. After each reasoning step $i$, we stop the thinking and force final answer generation with probability $1-\pi_\theta(y_{\le j}|x)$ by appending the end of reasoning token.

% \subsection{Why correctness classification is not enough}

% Several early-exit approaches train a classifier to estimate whether the current reasoning prefix is already sufficient, and then stop once the predicted probability exceeds a fixed threshold \cite{liu_answer_2025,jiang_flashthink_2025,zhang2025classifier}. These approaches exploit the useful fact that intermediate states contain information about answer quality. However, a fixed correctness threshold does not solve the optimal stopping problem. Its policy stops when $A(y_{\leq i} \mid x) \geq \gamma$ for some threshold $\gamma$. This policy treats $\gamma$ as a constant approximation to the state-dependent continuation value $C_\lambda(x, y_{\leq i})$ and can be arbitrarily suboptimal even with a perfect classifier as the following theorem shows.

Based on that analysis, we show that classification based on correctness can lose arbitrarily large value, even when the threshold $\gamma$ is chosen optimally:

\begin{theorem}[Threshold classification can have an arbitrarily large gap]
\label{thm:classification-arbitrary-gap}
For any \(\lambda\in(0,1)\) and any \(K>0\), there exists a finite-horizon CoT stopping
problem such that, even when \(A(y_{\le i}\mid x)\) is known exactly,
\[
V_\lambda(\pi^\star)
>
K\cdot \sup_{\gamma} V_\lambda(\pi_\gamma^{\mathrm{cls}}),
\]
where $V_\lambda(\pi)$ is the expected reward of policy $\pi$, \(\pi^\star\) is the optimal stopping policy and
\[
\pi_\gamma^{\mathrm{cls}}(y_{\le i}\mid x)
=
\mathbf{1}\!\left\{A(y_{\le i}\mid x) \ge \gamma\right\}
\]
is the best fixed-threshold correctness classifier.
\end{theorem}

The proof is deferred to Appendix \ref{appendix:proof}

\section{Implementation Details}\label{sec:implementation}

\subsection{Data Preprocessing Stage}

We demonstrate the effectiveness of our method on common math reasoning benchmarks. For a given reasoning model, we start by choosing a dataset of problems with a ground truth answer $\{(x_i,g_i)\}$. We then use vLLM \cite{kwon_efficient_2023} to generate reasoning traces $\{y_i\}$ for each problem in the dataset.

For each reasoning trace $y_i$, we manually intervene after each reasoning step, terminate the CoT, and prompt the model to give a final answer. We force the model to give only a final answer with no extra reasoning by prompting it to box its final answer and manually adding a box at the start of the answer block. This prevents a key issue with previous methodologies where the models kept thinking in the answer section~\citep{zhang_making_2025}, and is crucial to efficiently compute the answers. We then grade the answer against the ground truth using Math-Verify \cite{kydlicek_math-verify_2026} to compute the accuracy signals. This allows fast-training since the objective's expectation can be computed exactly, with no need to resample traces as the policy updates.

To efficiently compute the rewards, we use SGLang's \texttt{fork} functionality to make a tree of requests, branching into final answer computation and reasoning continuation after each paragraph.

% To process the large amount of requests efficiently and avoid repeatedly processing paragraphs shared between requests, we leverage SGLang's \texttt{fork} functionality powered by RadixAttention \cite{zheng_sglang_2024}. We make a tree of requests where each reasoning trace has a fork at the end of each paragraph, branching into a branch that continues reasoning and another that computes the final answer.

\subsection{Policy Architecture}

The stopping policy of OS-Pruner is parameterized as a linear head that takes the last hidden state $h_{l}$, which has been found to effectively encode the correctness of a partial CoT \citep{zhang2025classifier}, and outputs $\pi_{\theta}(y|x)=\sigma(W^\top h_{l}+b)$. Leveraging the base model's pretrained understanding and architecture designed for language processing, we further fine-tune its last $n$ self-attention layers, where $n$ is very small (we choose $n=2$ in all our experiments) to keep the policy lightweight.

We emphasize that we keep the base reasoning model weights frozen. The finetuned layer and additional linear layer at the end serve as an additional policy head, only invoked at the end of each paragraph at inference time. This is highly efficient on modern LLM serving frameworks such as vLLM, since each paragraph can be processed in parallel as a pre-fill (as opposed to autoregressively), and only $n=2$ layers of KV cache have to be recomputed. (Algorithm. \ref{alg:policy-inference})

\begin{algorithm}[h]
\caption{Inference with the Stopping Policy}
\label{alg:policy-inference}
\begin{algorithmic}[1]
\Require Prompt $x$; base LLM with $L$ layers (last $n$ fine-tuned); linear head $(W, b)$
\While{not stopped}
    \State Generate next token $x_{t+1} \sim p_{\text{LLM}}(\cdot \mid x_{1:t})$
    \If{$x_{t+1}$ ends a paragraph}
        \State $h \gets$ pre-fill last $n$ layers in parallel from cached layer-$(L{-}n)$ states
        \State $\pi_\theta \gets \sigma(W^\top h_l + b)$
        \State Stop with probability $\pi_\theta$
    \EndIf
\EndWhile
\State Elicit final answer from $x_{1:t}$
\end{algorithmic}
\end{algorithm}

We further note the training efficiency of our architecture. Since the base reasoning model is causal, the last hidden state corresponding to a paragraph is not affected by any subsequent tokens. Hence, we can process the entire reasoning trace in one forward pass and apply the linear layer to the relevant hidden states \textbf{in parallel}, allowing us to circumvent processing the reasoning traces one paragraph at a time.

\subsection{Progressive Training}

\begin{figure}
    \centering
    \includegraphics[width=1.0\linewidth]{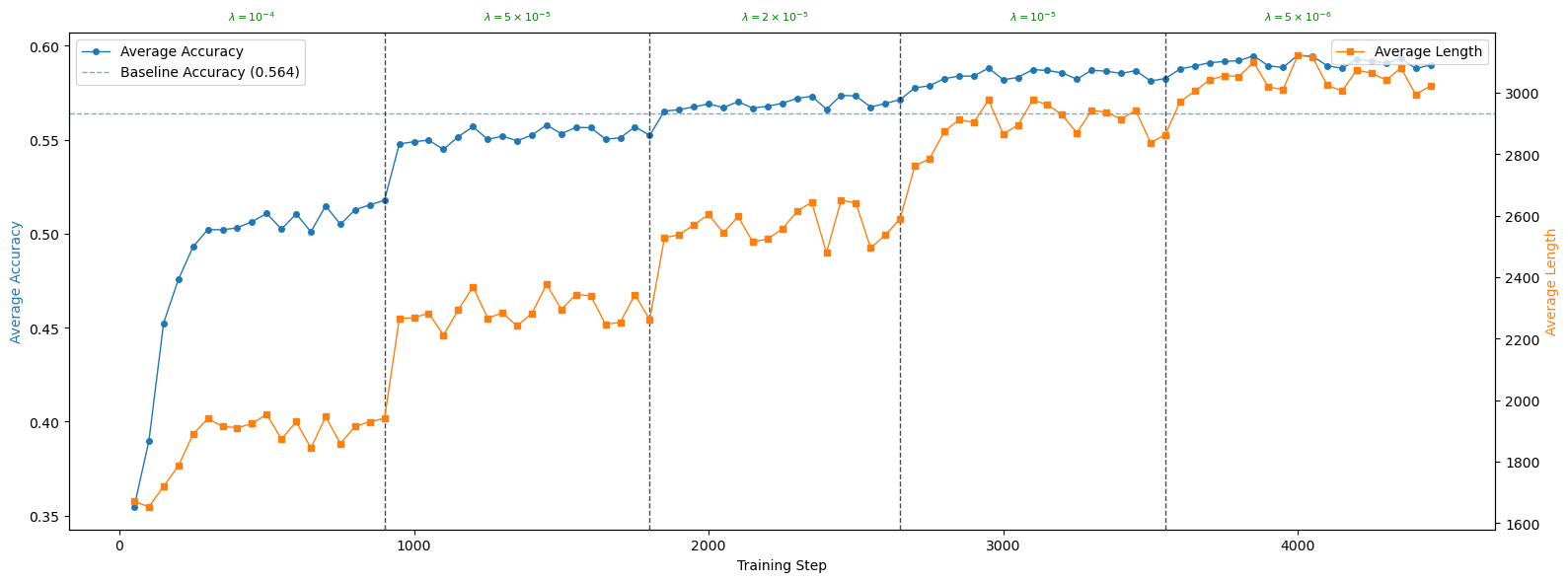}
    \caption{Training dynamics for DeepSeek-R1-Distill-Qwen-7B. The evaluations are conducted on a held-out validation subset of our training data with $500$ samples. The baseline length is 5182 tokens.}
    \label{fig:training dynamics}
\end{figure}

We empirically find that for relatively low values of $\lambda$, training from scratch gets stuck at a trivial local minimum where the policy continues to the end with very high probability. To address that issue, we start with a moderately high value of $\lambda$ and progressively anneal it during training at periodic points.

In Figure \ref{fig:training dynamics}, we study the training dynamics of our policies on the DeepSeek-R1-Distill-Qwen-7B model. Dashed lines show the points in training where $\lambda$ was annealed. Initially, the policy quickly learns to delay stopping and gain a major increase in accuracy. After each anneal, we see a jump in both accuracy and length. Finally, over the course of the training, the gap between accuracy and length diminishes as the $\lambda$ factor that trades them approaches $0$.

\section{Experiments}
We evaluate OS-Pruner with three goals. First, we test whether an optimal-stopping policy improves the accuracy-efficiency trade-off over training-based early-exit baselines. Second, we study the main empirical behavior of OS-Pruner across model families and benchmark difficulty levels. Third, we test whether OS-Pruner remains useful when appended to a model that has already been optimized for concise reasoning through a model-side method.

\textbf{Models.} We experiment with 3 models:
\begin{itemize}
    \item DeepSeek-R1-Distill-Qwen-7B~\citep{deepseek-ai_deepseek-r1_2025}: a popular open-weight reasoning model for reasoning and early-stopping benchmarks.
    \item GPT-OSS-20B~\citep{openai_gpt-oss-120b_2025}: a stronger and more concise larger-scale reasoning model, allowing us to evaluate whether OS-Pruner remains useful when the base model already produces shorter traces. 
    \item DRPO-7B: obtained from DeepSeek-R1-Distill-Qwen-7B using DRPO~\cite{li_drpo_2026} using their heaviest length penalty. DRPO is a state of the art model-side efficient-reasoning method that substantially reduces overthinking and even improves accuracy. This model allows us to test whether OS-Pruner can still improve upon model-side compression.
\end{itemize}

\textbf{Datasets.} We train our policies on the DeepScaleR-Preview-Dataset~\citep{deepscaler2025}, consisting of approximately 40000 problems sourced from AIME (1984-2023), AMC (prior to 2023), Omni-MATH~\citep{gao_omni-math_2024}, and Still~\citep{min_imitate_2024}. We hold out $500$ problems for validation. We then evaluate them on held-out mathematical reasoning benchmarks with various levels of difficulty: GSM8K~\cite{cobbe_training_2021} for easy grade-school arithmetic, MATH-500~\cite{hendrycks_measuring_2021} for medium-difficulty competition mathematics, and AIME24/25 (aggregated) for challenging olympiad-style problems.

\textbf{Policy training.} For each base model, we first generate full reasoning traces on the training set and then construct stopping labels by terminating the trace after each reasoning step and forcing the model to produce only a final answer. We use Math-Verify~\cite{kydlicek_math-verify_2026} to compare extracted answers against ground truth. Following the method described in Section~\ref{sec:implementation}, the base reasoning model is kept frozen, only the stopping head and the last two transformer layers used by the policy are fine-tuned. On 4 A100 GPUs, this dataset preparation step takes approximately 4, 2.5, and 1 day for DeepSeek-R1-Distill-Qwen-7B, GPT-OSS-20B, and DRPO-7B respectively.
We cap generation at $15{,}000$ tokens and remove training examples whose original trace does not terminate within this budget.
For DeepSeek-R1-Distill-Qwen-7B, we train policies with
\[
\lambda \in \{10^{-4}, 5{\times}10^{-5}, 2{\times}10^{-5}, 10^{-5}, 5{\times}10^{-6}, 2{\times}10^{-6}, 10^{-6}, 0\}.
\]
For GPT-OSS-20B and DRPO-7B, we train policies with
\[
\lambda \in \{2{\times}10^{-4}, 10^{-4}, 5{\times}10^{-5}, 2{\times}10^{-5}, 10^{-5}, 0\}.
\]

For optimization, we use AdamW with learning rate $10^{-5}$, $(\beta_1,\beta_2)=(0.9,0.999)$, and $\epsilon=10^{-8}$.

\textbf{Evaluation.}
We evaluate accuracy using Pass@1, averaged over 16 sampled solutions for each problem, and
measure efficiency using the average number of generated reasoning tokens before final-answer
elicitation. To summarize the accuracy--efficiency trade-off with a single scalar, we additionally
report the Accuracy Efficiency Score (AES), which is used in the efficient reasoning literature ~\citep{luo_o1-pruner_2025, li_drpo_2026, ning2025thoughtsequal}. AES is defined as

\begin{equation}
\text{AES} = 
\begin{cases} 
\alpha \cdot \Delta\text{Length} + \beta \cdot |\Delta\text{Acc}|, & \text{if } \Delta\text{Acc} \geq 0 \\ 
\alpha \cdot \Delta\text{Length} - \gamma \cdot |\Delta\text{Acc}|, & \text{if } \Delta\text{Acc} < 0, 
\end{cases}
\end{equation}

where $\Delta \text{Length}=\frac{\text{Length}_{\text{baseline}}-\text{Length}_{\text{method}}}{\text{Length}_{\text{baseline}}}$ and $\Delta \text{Acc}=\frac{\text{Acc}_{\text{method}}-\text{Acc}_{\text{baseline}}}{\text{Acc}_{\text{baseline}}}$. We use the original $\alpha=1$ and $\beta=3$ values but raise $\gamma$ to $7$ to emphasize the importance of preserving accuracy. For each method, we report the stopping policy with the best AES.

For DeepSeek-R1-Distill-Qwen-7B and DRPO-7B, we set temperature $=0.6$ and top-$p=0.95$.
For GPT-OSS-20B, we set temperature $=1.0$ and top-$p=1.0$. Matching training, we cap
generation at $15{,}000$ tokens. After stopping, we append a boxed-answer prefix at the start
of the answer section to force the model to produce a final answer instead of continuing to reason
in the answer section; we ablate this choice in Appendix~\ref{appendix:ablation}.

\textbf{Baselines.} We compare against three training-based early stopping methods:

\begin{itemize}
    \item \textbf{Classification Probe\cite{zhang2025classifier}:} 
    trains a stopping policy to predict whether the current partial CoT leads to a correct answer; generation stops when the classifier score exceeds a fixed threshold. For fair comparison, we train the classifier with the same architecture as OS-Pruner.
    \item \textbf{Answer Convergence \cite{liu_answer_2025}:} trains a classifier to predict whether the answer produced from the current prefix will remain unchanged until the end of the original trace. While the original paper uses an LSTM for their classifier architecture, we scale it to use the OS-Pruner architecture.
    \item \textbf{FlashThink \cite{jiang_flashthink_2025}}: uses a verifier-style LLM to judge whether the current reasoning is sufficient. Following their experiments, we use the Qwen2.5-7B-Instruct model, but for fairness and training efficiency, we fine-tune only the final two layers of the verifier model rather than the full 7B model used in the original work.
\end{itemize}

We test all baselines with all inference thresholds $\in [0.7, 0.75, 0.8, 0.85,0.9]$.

\subsection{Optimal stopping improves the accuracy efficiency tradeoff}
Figure~\ref{fig:overall_pareto_frontiers} shows the accuracy-length trade-off for all models and benchmarks. OS-Pruner consistently lies on or near the Pareto frontier, while fixed-threshold early-exit baselines degrade more sharply as the length budget is reduced.
In the low accuracy-sacrifice regime (<2\%), OS-Pruner reduces reasoning length by roughly 20-60\% on most model-dataset pairs while preserving accuracy. The strongest gains occur on easier or medium-difficulty datasets, where overthinking is most prevalent. For example, with DeepSeek-R1-Distill-Qwen-7B, OS-Pruner reduces length by 59.3\% on GSM8K and 52.8\% on MATH-500 at the best-AES operating point, while changing accuracy by only $-0.7$ and $+2.7$ percentage points, respectively. On AIME, the same policy is more conservative, reducing length by only 6.9\% at the best-AES point, which is expected since extra computation is valuable on harder tasks. Table~\ref{tab:AES table} summarizes the best operating point for each method according to the AES.

\begin{figure*}[t]
    \centering
    
    % Row 1: DeepSeek Model
    \begin{subfigure}{\textwidth}
        \centering
        \includegraphics[width=\linewidth]{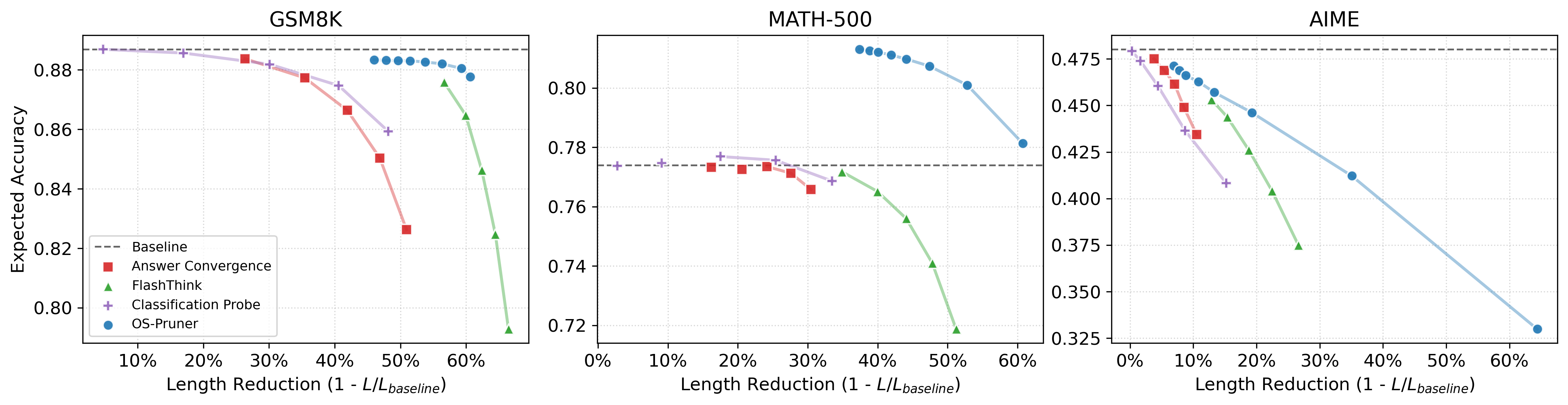}
        \caption{DeepSeek-R1-Distill-Qwen-7B}
        \label{fig:model_deepseek}
    \end{subfigure}
    
    \vspace{0.5em} % Vertical space between rows
    
    % Row 2: GPT-OSS Model
    \begin{subfigure}{\textwidth}
        \centering
        \includegraphics[width=\linewidth]{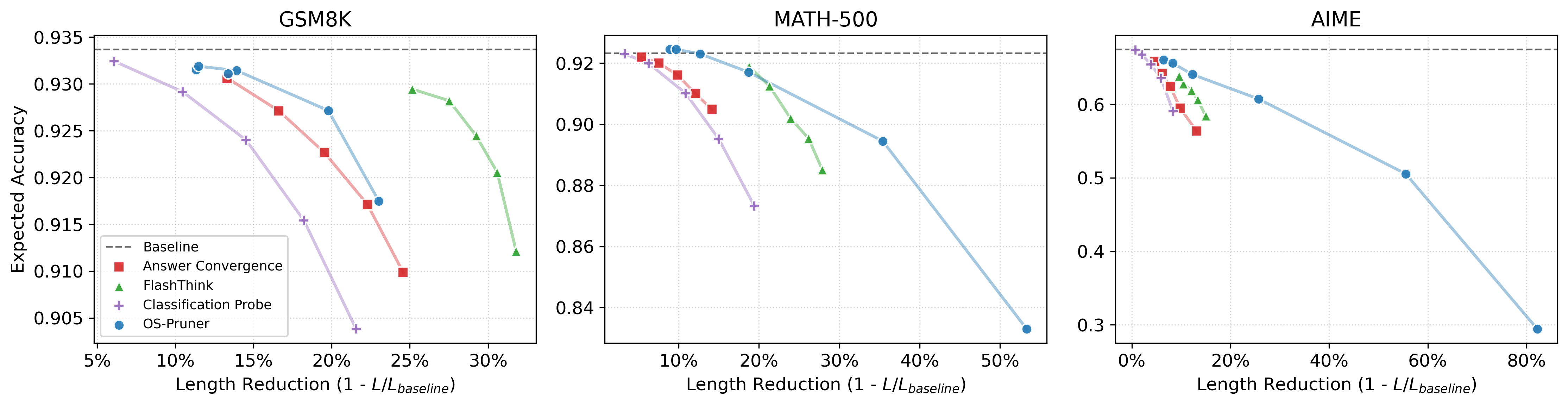}
        \caption{GPT-OSS-20B}
        \label{fig:model_gpt_oss}
    \end{subfigure}
    
    \vspace{0.5em}
    
    % Row 3: DRPO Model
    \begin{subfigure}{\textwidth}
        \centering
        \includegraphics[width=\linewidth]{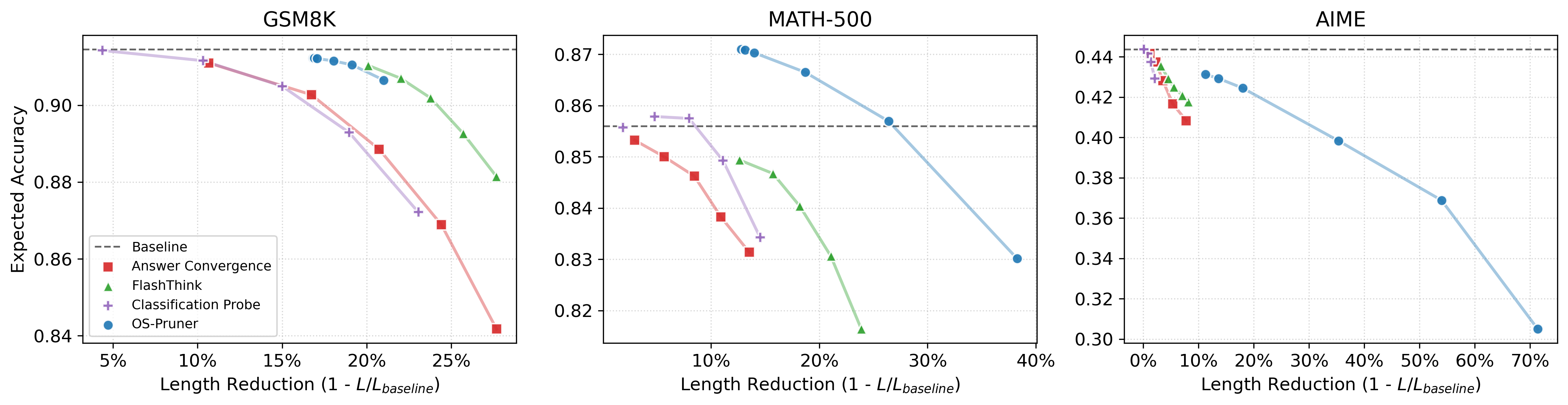}
        \caption{DRPO-7B}
        \label{fig:model_drpo}
    \end{subfigure}
    
    \caption{Pareto Frontiers of Expected Accuracy vs. Length Reduction across various reasoning-heavy datasets and models. }
    \label{fig:overall_pareto_frontiers}
\end{figure*}

\begin{table}
    \centering
    \resizebox{\textwidth}{!}{%
        \begin{tabular}{lcccccccccccc}
        \toprule
        \multirow{2}{*}{Model / Method} & \multicolumn{3}{c}{GSM8K} & \multicolumn{3}{c}{MATH-500} & \multicolumn{3}{c}{AIME} & \multicolumn{3}{c}{\textit{AVERAGE}} \\
        \cmidrule(lr){2-4} \cmidrule(lr){5-7} \cmidrule(lr){8-10} \cmidrule(lr){11-13}
         & Acc & Len & AES & Acc & Len & AES & Acc & Len & AES & Acc & Len & AES \\
        \midrule
        \multicolumn{13}{l}{\textbf{DeepSeek-R1-Distill-Qwen-7B}} \\
        Baseline & 88.7 & 1219 & 0.00 & 77.4 & 3377 & 0.00 & 48.0 & 9873 & 0.00 & 71.4 & 4823 & 0.00 \\
        Classification Probe & 87.5 & 724 & 0.34 & \underline{76.9} & 2247 & 0.31 & \textbf{47.9} & 9846 & -0.00 & \underline{70.8} & 4272 & 0.22 \\
        Answer Convergence & 86.6 & 708 & 0.32 & 76.6 & 2348 & 0.26 & \underline{47.5} & 9500 & \underline{0.01} & 70.2 & 4185 & 0.20 \\
        FlashThink & \underline{87.6} & \underline{528} & \underline{0.51} & 76.5 & \underline{2027} & \underline{0.36} & 45.3 & \textbf{8602} & -0.01 & 69.8 & \underline{3719} & \underline{0.29} \\
        OS-Pruner & \textbf{88.0} & \textbf{496} & \textbf{0.56} & \textbf{78.1} & \textbf{1325} & \textbf{0.63} & 47.1 & \underline{9190} & \textbf{0.02} & \textbf{71.1} & \textbf{3671} & \textbf{0.40} \\
        \midrule
        \multicolumn{13}{l}{\textbf{GPT-OSS-20B}} \\
        Baseline & 93.4 & 278 & 0.00 & 92.3 & 1126 & 0.00 & 67.5 & 6001 & 0.00 & 84.4 & 2468 & 0.00 \\
        Classification Probe & 92.4 & 237 & 0.10 & \textbf{92.0} & 1055 & 0.05 & \textbf{67.4} & 5958 & \textbf{0.00} & \textbf{83.9} & 2417 & 0.05 \\
        Answer Convergence & 92.3 & 223 & 0.14 & 91.6 & 1015 & 0.06 & 65.8 & 5723 & -0.04 & \underline{83.2} & 2321 & 0.06 \\
        FlashThink & \textbf{92.8} & \textbf{201} & \textbf{0.25} & \underline{91.9} & \underline{915} & \underline{0.17} & 63.9 & \textbf{5426} & -0.09 & 82.9 & \textbf{2181} & \underline{0.11} \\
        OS-Pruner & \underline{92.7} & \underline{223} & \underline{0.17} & 89.4 & \textbf{728} & \textbf{0.21} & \underline{66.0} & \underline{5617} & \underline{-0.01} & 82.7 & \underline{2189} & \textbf{0.12} \\
        \midrule
        \multicolumn{13}{l}{\textbf{DRPO-7B}} \\
        Baseline & 91.5 & 398 & 0.00 & 85.6 & 1096 & 0.00 & 44.4 & 4359 & 0.00 & 73.8 & 1951 & 0.00 \\
        Classification Probe & 90.5 & 338 & 0.10 & \textbf{85.8} & 1009 & 0.08 & \textbf{44.4} & 4354 & 0.00 & \textbf{73.5} & 1900 & 0.06 \\
        Answer Convergence & 90.3 & 331 & 0.11 & 84.6 & 1004 & 0.04 & \underline{44.2} & 4307 & \underline{0.00} & \underline{73.0} & 1881 & 0.05 \\
        FlashThink & \underline{90.7} & \textbf{310} & \textbf{0.18} & 84.7 & \underline{924} & \underline{0.11} & 43.5 & \underline{4221} & -0.01 & 73.0 & \underline{1818} & \underline{0.09} \\
        OS-Pruner & \textbf{91.0} & \underline{322} & \underline{0.17} & \underline{85.7} & \textbf{807} & \textbf{0.27} & 36.9 & \textbf{2006} & \textbf{0.17} & 71.2 & \textbf{1045} & \textbf{0.20} \\
        \bottomrule
        \end{tabular}
    }
    \caption{Summary of the stopping policy with best AES for each method. In each group, the best result is \textbf{bolded} and the second best is \underline{underlined}.}
    \label{tab:AES table}
\end{table}

\section{Discussion}

With the vast variety of efficient reasoning approaches, we conclude by summarizing OS-Pruner's placement in the literature and relation to previous approaches. Within training-based early stopping methods, we show both theoretically and empirically that our optimal stopping formulation surpasses output-side methods in terms of trading accuracy for efficiency.  Training-free approaches, with more handcrafted stopping conditions, can be used to augment our approach, as their stopping criteria can be used as input \emph{features} in the training of our stopping policies.

Budget-based methods solve a different problem within the efficiency framework: given a problem and a fixed constrained budget, maximize the chances of producing a correct answer. Our approach operates in the dynamic reasoning regime where more thinking is decided based on the problem and partial reasoning so far.

Finally, while model-side compression methods demonstrate great performance, their learning entangles several behaviors: writing more concisely, changing the search strategy, avoiding dead ends, and deciding when to terminate. Learning to optimally stop only from examples of complete reasoning traces and their final rewards can be difficult. In contrast, our approach features fine-grained reward labeling after each reasoning step and a plug-in policy trained solely with a stopping objective. Our experiments on DRPO-7B demonstrate that our approach can provide additional value when used together with a leading model-side compression method.

\textbf{Broader Impact:} this work aims to make reasoning language models more efficient by reducing unnecessary chain-of-thought generation while preserving task accuracy. By lowering inference latency and token usage, OS-Pruner can reduce the computational cost and energy footprint of deploying reasoning models, and may make such models more accessible in settings with limited compute budgets. The method also gives practitioners explicit control over the trade-off between reasoning length and accuracy through the stopping-cost parameter $\lambda$, which can help adapt deployment to different resource and reliability requirements. Our work can further help democratize artificial intelligence by improving the efficiency and generation quality of smaller open-source models that suffer most from overthinking.

\textbf{Limitations \& Future Work:} while we have demonstrated the success of our approach on mathematical reasoning tasks, future work can explore more general tasks such as coding and scientific reasoning. Furthermore, scaling the experiments beyond 20B parameters is crucial to understand if our approach can improve on frontier models. Finally, while we focus our efforts on fair scientific comparison to previous work, exploration along the architecture and input features dimensions can potentially improve our method's absolute performance.

% \section{Detailed Comparison to Previous Work}

% (rough idea)

% We start by discussing the early-exit training-based methods we use as baselines. All of them rely on classification, and to the best of our knowledge, we are the first to uncover the optimal stopping framing of the problem. If we're doing any theory, this could be the place: how much can we prove about optimal stopping outperforming classification?

% Training-free methods exist too (give examples, Adam found good ones). The signals they use can be leveraged and combined as policy input features in our formulation.

% Other methods such as budget forcing and budget guidance impose a fixed generation budget, but our approach allows the reasoning length to dynamically very with difficulty as needed.

% Finally, model-side training methods have shown immense success. However, learning optimal stopping with pure RL or SFT is difficult because there are multiple axes of improvement. These methods could improve efficiency by removing filler words, improving the model's ability to kill/not explore dead-ends, etc. and stopping is only one part of the equation. Our method is laser-focused on stopping and explicitly computes reward signals after every single reasoning step, making efficient use of the data and providing fine-grained annotations for stopping. We show in a separate section that DRPO, a leading model-side method, despite massively cutting the length, still has a lot of room to grow in stopping that we can better fill.
\newpage

\bibliographystyle{unsrt}
\bibliography{CoT}

%%%%%%%%%%%%%%%%%%%%%%%%%%%%%%%%%%%%%%%%%%%%%%%%%%%%%%%%%%%%
\appendix

\section{Ablation of Final Answer Forcing}\label{appendix:ablation}

\begin{figure*}[h]
    \centering
    
    % Row 1: DeepSeek Model
    \begin{subfigure}{\textwidth}
        \centering
        \includegraphics[width=\linewidth]{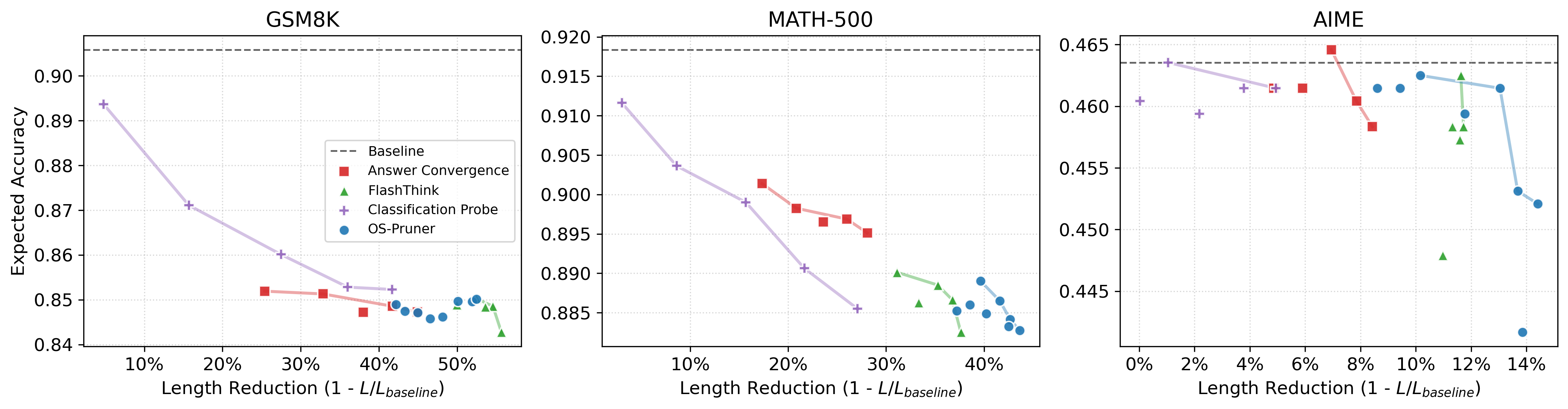}
        \caption{DeepSeek-R1-Distill-Qwen-7B}
        \label{fig:model_deepseek}
    \end{subfigure}
    
    \vspace{0.5em} % Vertical space between rows
    
    % Row 2: GPT-OSS Model
    \begin{subfigure}{\textwidth}
        \centering
        \includegraphics[width=\linewidth]{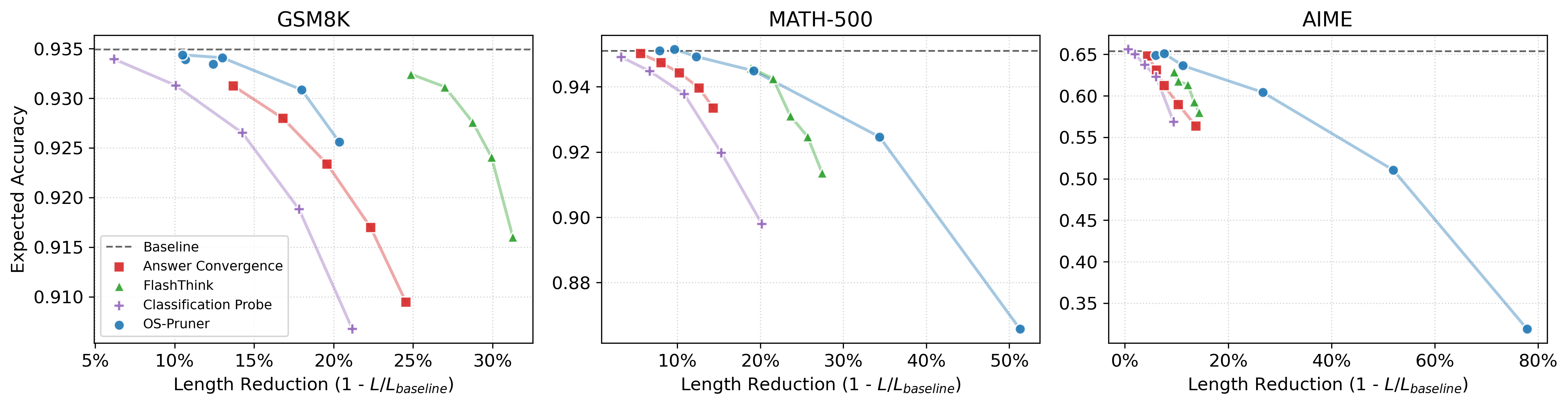}
        \caption{GPT-OSS-20B}
        \label{fig:model_gpt_oss}
    \end{subfigure}
    
    \vspace{0.5em}
    
    % Row 3: DRPO Model
    \begin{subfigure}{\textwidth}
        \centering
        \includegraphics[width=\linewidth]{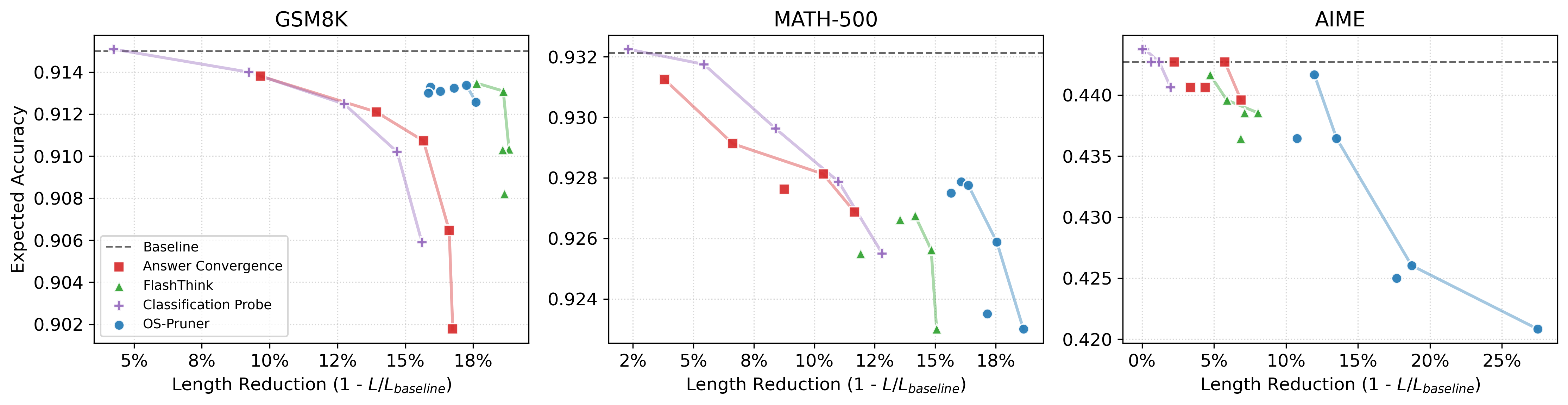}
        \caption{DRPO-7B}
        \label{fig:model_drpo}
    \end{subfigure}
    
    \caption{Pareto Frontiers of Expected Accuracy vs. Length Reduction when the models are allowed to think freely in the answer section.}
    \label{fig:full_pareto_frontiers}
\end{figure*}

We ablate the choice of forcing a final answer by appending a box in the answer section and let the model behave freely. We accumulate the length of the reasoning and answer sections and report our results in Figure. \ref{fig:full_pareto_frontiers}.

While our main results from Figure \ref{fig:overall_pareto_frontiers} transfer seamlessly for GPT-OSS-20B, and while OS-Pruner's length reduction still significantly exceed the baselines, we observe a weaker ability to control the accuracy-length trade-off and trace a frontier for all the methods tested on DRPO-7B, and more severely, on DeepSeek-R1-Distill-Qwen-7B. We find that to be mainly due to the poor instruction following capabilities of these models. By further investigating the lengths of the reasoning and answer sections separately, we find that when the CoTs are terminated very early, these models compensate by thinking in the answer section, making the trade-off difficult (see App. \ref{appendix:full tables}). A similar finding was observed in~\citep{zhang_making_2025}. In contrast, the GPT-OSS-20B model consistently outputs only a final answer when instructed to do so, even if its thinking is terminated early, due to better instruction tuning.

Notably, we observe a higher baseline accuracy on the MATH-500 benchmark. Investigating further, we find that the benchmark includes many problems with non-numerical answers, where the models often only reason enough to produce an answer by thinking further in the answer section. For example, in a problem about polynomial multiplication, DeepSeek-R1-Distill-Qwen-7B writes a reasoning trace reminding itself how to multiply polynomials and how it has to expand and match terms, but does not arrive at the final answer.

\section{Proof of Theorem \ref{thm:classification-arbitrary-gap}}\label{appendix:proof}
\begin{proof}
Fix $\lambda \in (0,1)$ and $K>0$. Choose
\[
    0 < \varepsilon < \frac{1-\lambda}{2K}.
\]
We construct a finite-horizon stopping problem with one decision point followed,
if the policy continues, by forced termination.

At the first decision point, the process is in one of two states,
$s_g$ or $s_b$, each with probability $1/2$. Both states have zero current
reasoning length and the same current correctness estimate:
\[
    L(s_g)=L(s_b)=0,
    \qquad
    A(s_g)=A(s_b)=\varepsilon.
\]
Thus stopping immediately in either state gives reward $\varepsilon$.

If the policy continues from $s_g$, the process generates one additional token
and reaches a terminal state $t_g$ with correctness $A(t_g)=1$ and length
$L(t_g)=1$. Hence the reward of continuing from $s_g$ is
\[
    A(t_g)-\lambda L(t_g)
    =
    1-\lambda.
\]
Since $\lambda<1$ and $\varepsilon < 1-\lambda$, continuing is strictly better
than stopping in $s_g$.

If the policy continues from $s_b$, the process generates $M$ additional tokens
and reaches a terminal state $t_b$ with correctness $A(t_b)=0$ and length
$L(t_b)=M$. Hence the utility of continuing from $s_b$ is
\[
    A(t_b)-\lambda L(t_b)
    =
    -\lambda M,
\]
which is strictly worse than stopping in $s_b$. Therefore the optimal policy
continues in $s_g$ and stops in $s_b$, giving value
\[
    V_\lambda(\pi^\star)
    =
    \frac12(1-\lambda) + \frac12 \varepsilon.
\]

Now consider any fixed-threshold correctness policy $\pi^{\mathrm{cls}}_\gamma$.
Since
\[
    A(s_g)=A(s_b)=\varepsilon,
\]
the threshold policy must take the same action in both states.

If $\gamma \le \varepsilon$, then the policy stops in both states, and its value is
\[
    V_\lambda(\pi^{\mathrm{cls}}_\gamma)=\varepsilon.
\]
If $\gamma > \varepsilon$, then the policy continues in both states, and its value is
\[
    V_\lambda(\pi^{\mathrm{cls}}_\gamma)
    =
    \frac12(1-\lambda) + \frac12(-\lambda M)
    =
    \frac{1-\lambda-\lambda M}{2}.
\]
Choose $M$ large enough that
\[
    \frac{1-\lambda-\lambda M}{2} \le \varepsilon,
\]
for example any integer
\[
    M \ge \frac{1-\lambda-2\varepsilon}{\lambda}.
\]
Then every threshold policy has value at most $\varepsilon$, and the best
threshold policy obtains exactly
\[
    \sup_\gamma V_\lambda(\pi^{\mathrm{cls}}_\gamma)=\varepsilon.
\]

Therefore,
\[
    \frac{V_\lambda(\pi^\star)}
    {\sup_\gamma V_\lambda(\pi^{\mathrm{cls}}_\gamma)}
    =
    \frac{\frac12(1-\lambda)+\frac12\varepsilon}{\varepsilon}
    =
    \frac{1-\lambda}{2\varepsilon}+\frac12
    >
    K,
\]
by the choice of $\varepsilon$. This proves the claim.
\end{proof}

%%%%%%%%%%%%%%%%%%%%%%%%%%%%%%%%%%%%%%%%%%%%%%%%%%%%%%%%%%%%

\newpage

\section{Full Evaluation Tables}
\label{appendix:full tables}

\begin{table*}[htbp]
    \centering
    \resizebox{\textwidth}{!}{%
        \begin{tabular}{lcccccccc}
        \toprule
        \multicolumn{9}{c}{\textbf{DeepSeek-R1-Distill-Qwen-7B}} \\
        \multirow{2}{*}{Method / Param} & \multicolumn{2}{c}{GSM8K} & \multicolumn{2}{c}{MATH-500} & \multicolumn{2}{c}{AIME} & \multicolumn{2}{c}{\textit{AVERAGE}} \\
        \cmidrule(lr){2-3} \cmidrule(lr){4-5} \cmidrule(lr){6-7} \cmidrule(lr){8-9}
         & Acc & Length (Red.) & Acc & Length (Red.) & Acc & Length (Red.) & Acc & Length (Red.) \\
        \midrule
        \textit{Baseline} \\
        \quad N/A & 88.7 & 1219 (0.0\%) & 77.4 & 3377 (0.0\%) & 48.0 & 9873 (0.0\%) & 71.4 & 4823 (0.0\%) \\
        \addlinespace[0.2ex]
        \textit{Classification Probe} \\
        \quad $\gamma=0.70$ & 85.9 & 632 (48.2\%) & 76.9 & 2247 (33.5\%) & 40.8 & 8374 (15.2\%) & 67.9 & 3751 (32.3\%) \\
        \quad $\gamma=0.75$ & 87.5 & 724 (40.6\%) & 77.6 & 2518 (25.4\%) & 43.6 & 9013 (8.7\%) & 69.6 & 4085 (24.9\%) \\
        \quad $\gamma=0.80$ & 88.2 & 852 (30.1\%) & 77.7 & 2786 (17.5\%) & 46.0 & 9438 (4.4\%) & 70.6 & 4358 (17.3\%) \\
        \quad $\gamma=0.85$ & 88.6 & 1012 (17.0\%) & 77.5 & 3070 (9.1\%) & 47.4 & 9713 (1.6\%) & 71.1 & 4598 (9.2\%) \\
        \quad $\gamma=0.90$ & 88.7 & 1161 (4.7\%) & 77.4 & 3282 (2.8\%) & 47.9 & 9846 (0.3\%) & 71.3 & 4763 (2.6\%) \\
        \addlinespace[0.2ex]
        \textit{Answer Convergence} \\
        \quad $\gamma=0.70$ & 82.6 & 598 (50.9\%) & 76.6 & 2348 (30.5\%) & 43.4 & 8833 (10.5\%) & 67.5 & 3926 (30.6\%) \\
        \quad $\gamma=0.75$ & 85.0 & 648 (46.9\%) & 77.1 & 2445 (27.6\%) & 44.9 & 9033 (8.5\%) & 69.0 & 4042 (27.7\%) \\
        \quad $\gamma=0.80$ & 86.6 & 708 (41.9\%) & 77.3 & 2561 (24.2\%) & 46.1 & 9179 (7.0\%) & 70.0 & 4149 (24.4\%) \\
        \quad $\gamma=0.85$ & 87.7 & 786 (35.5\%) & 77.2 & 2681 (20.6\%) & 46.9 & 9338 (5.4\%) & 70.6 & 4268 (20.5\%) \\
        \quad $\gamma=0.90$ & 88.4 & 898 (26.3\%) & 77.3 & 2829 (16.2\%) & 47.5 & 9500 (3.8\%) & 71.1 & 4409 (15.4\%) \\
        \addlinespace[0.2ex]
        \textit{FlashThink} \\
        \quad $\gamma=0.70$ & 79.3 & 409 (66.5\%) & 71.9 & 1647 (51.2\%) & 37.5 & 7246 (26.6\%) & 62.9 & 3101 (48.1\%) \\
        \quad $\gamma=0.75$ & 82.5 & 433 (64.5\%) & 74.1 & 1763 (47.8\%) & 40.4 & 7658 (22.4\%) & 65.7 & 3285 (44.9\%) \\
        \quad $\gamma=0.80$ & 84.6 & 458 (62.4\%) & 75.6 & 1888 (44.1\%) & 42.6 & 8020 (18.8\%) & 67.6 & 3456 (41.7\%) \\
        \quad $\gamma=0.85$ & 86.5 & 488 (59.9\%) & 76.5 & 2027 (40.0\%) & 44.4 & 8356 (15.4\%) & 69.1 & 3624 (38.4\%) \\
        \quad $\gamma=0.90$ & 87.6 & 528 (56.6\%) & 77.2 & 2198 (34.9\%) & 45.3 & 8602 (12.9\%) & 70.0 & 3776 (34.8\%) \\
        \addlinespace[0.2ex]
        \textit{OS-Pruner} \\
        \quad $\lambda=0$ & 88.3 & 658 (46.0\%) & 81.3 & 2114 (37.4\%) & 47.1 & 9190 (6.9\%) & 72.3 & 3988 (30.1\%) \\
        \quad $\lambda=1.0 \times 10^{-6}$ & 88.3 & 636 (47.8\%) & 81.3 & 2064 (38.9\%) & 46.9 & 9105 (7.8\%) & 72.2 & 3935 (31.5\%) \\
        \quad $\lambda=2.0 \times 10^{-6}$ & 88.3 & 614 (49.6\%) & 81.2 & 2023 (40.1\%) & 46.6 & 9007 (8.8\%) & 72.0 & 3881 (32.8\%) \\
        \quad $\lambda=5.0 \times 10^{-6}$ & 88.3 & 591 (51.5\%) & 81.1 & 1961 (41.9\%) & 46.3 & 8805 (10.8\%) & 71.9 & 3786 (34.7\%) \\
        \quad $\lambda=1.0 \times 10^{-5}$ & 88.3 & 563 (53.8\%) & 81.0 & 1887 (44.1\%) & 45.7 & 8561 (13.3\%) & 71.6 & 3670 (37.1\%) \\
        \quad $\lambda=2.0 \times 10^{-5}$ & 88.2 & 532 (56.3\%) & 80.7 & 1776 (47.4\%) & 44.6 & 7969 (19.3\%) & 71.2 & 3426 (41.0\%) \\
        \quad $\lambda=5.0 \times 10^{-5}$ & 88.0 & 496 (59.3\%) & 80.1 & 1593 (52.8\%) & 41.2 & 6412 (35.1\%) & 69.8 & 2834 (49.0\%) \\
        \quad $\lambda=1.0 \times 10^{-4}$ & 87.8 & 480 (60.6\%) & 78.1 & 1325 (60.8\%) & 33.0 & 3516 (64.4\%) & 66.3 & 1774 (61.9\%) \\
        \addlinespace[0.2ex]
        \bottomrule
        \end{tabular}
    }
    \caption{Full results for DeepSeek-R1-Distill-Qwen-7B with forced answer.}
    \label{tab:full_deepseek_r1_distill_qwen_7b_forced_answer}
\end{table*}

\begin{table*}[htbp]
    \centering
    \resizebox{\textwidth}{!}{%
        \begin{tabular}{lcccccccc}
        \toprule
        \multicolumn{9}{c}{\textbf{GPT-OSS-20B}} \\
        \multirow{2}{*}{Method / Param} & \multicolumn{2}{c}{GSM8K} & \multicolumn{2}{c}{MATH-500} & \multicolumn{2}{c}{AIME} & \multicolumn{2}{c}{\textit{AVERAGE}} \\
        \cmidrule(lr){2-3} \cmidrule(lr){4-5} \cmidrule(lr){6-7} \cmidrule(lr){8-9}
         & Acc & Length (Red.) & Acc & Length (Red.) & Acc & Length (Red.) & Acc & Length (Red.) \\
        \midrule
        \textit{Baseline} \\
        \quad N/A & 93.4 & 278 (0.0\%) & 92.3 & 1126 (0.0\%) & 67.5 & 6001 (0.0\%) & 84.4 & 2468 (0.0\%) \\
        \addlinespace[0.2ex]
        \textit{Classification Probe} \\
        \quad $\gamma=0.70$ & 90.4 & 218 (21.6\%) & 87.3 & 907 (19.4\%) & 59.1 & 5501 (8.3\%) & 78.9 & 2209 (16.4\%) \\
        \quad $\gamma=0.75$ & 91.5 & 227 (18.2\%) & 89.5 & 957 (15.0\%) & 63.5 & 5647 (5.9\%) & 81.5 & 2277 (13.0\%) \\
        \quad $\gamma=0.80$ & 92.4 & 237 (14.5\%) & 91.0 & 1004 (10.9\%) & 65.4 & 5768 (3.9\%) & 82.9 & 2336 (9.8\%) \\
        \quad $\gamma=0.85$ & 92.9 & 249 (10.5\%) & 92.0 & 1055 (6.3\%) & 66.8 & 5880 (2.0\%) & 83.9 & 2395 (6.3\%) \\
        \quad $\gamma=0.90$ & 93.2 & 261 (6.1\%) & 92.3 & 1089 (3.3\%) & 67.4 & 5958 (0.7\%) & 84.3 & 2436 (3.4\%) \\
        \addlinespace[0.2ex]
        \textit{Answer Convergence} \\
        \quad $\gamma=0.70$ & 91.0 & 209 (24.6\%) & 90.5 & 967 (14.1\%) & 56.4 & 5214 (13.1\%) & 79.3 & 2130 (17.3\%) \\
        \quad $\gamma=0.75$ & 91.7 & 216 (22.3\%) & 91.0 & 990 (12.1\%) & 59.5 & 5413 (9.8\%) & 80.7 & 2206 (14.7\%) \\
        \quad $\gamma=0.80$ & 92.3 & 223 (19.5\%) & 91.6 & 1015 (9.9\%) & 62.4 & 5534 (7.8\%) & 82.1 & 2257 (12.4\%) \\
        \quad $\gamma=0.85$ & 92.7 & 231 (16.6\%) & 92.0 & 1041 (7.5\%) & 64.2 & 5634 (6.1\%) & 83.0 & 2302 (10.1\%) \\
        \quad $\gamma=0.90$ & 93.1 & 241 (13.3\%) & 92.2 & 1065 (5.4\%) & 65.8 & 5723 (4.6\%) & 83.7 & 2343 (7.8\%) \\
        \addlinespace[0.2ex]
        \textit{FlashThink} \\
        \quad $\gamma=0.70$ & 91.2 & 189 (31.8\%) & 88.5 & 812 (27.9\%) & 58.4 & 5104 (15.0\%) & 79.4 & 2035 (24.9\%) \\
        \quad $\gamma=0.75$ & 92.1 & 193 (30.6\%) & 89.5 & 831 (26.2\%) & 60.6 & 5201 (13.3\%) & 80.7 & 2075 (23.4\%) \\
        \quad $\gamma=0.80$ & 92.4 & 196 (29.2\%) & 90.2 & 857 (24.0\%) & 61.9 & 5278 (12.0\%) & 81.5 & 2110 (21.7\%) \\
        \quad $\gamma=0.85$ & 92.8 & 201 (27.5\%) & 91.2 & 886 (21.3\%) & 62.8 & 5376 (10.4\%) & 82.3 & 2155 (19.7\%) \\
        \quad $\gamma=0.90$ & 92.9 & 208 (25.2\%) & 91.9 & 915 (18.8\%) & 63.9 & 5426 (9.6\%) & 82.9 & 2183 (17.8\%) \\
        \addlinespace[0.2ex]
        \textit{OS-Pruner} \\
        \quad $\lambda=0$ & 93.2 & 246 (11.3\%) & 92.5 & 1026 (8.9\%) & 66.0 & 5617 (6.4\%) & 83.9 & 2296 (8.9\%) \\
        \quad $\lambda=1.0 \times 10^{-5}$ & 93.2 & 246 (11.5\%) & 92.5 & 1017 (9.7\%) & 65.6 & 5506 (8.2\%) & 83.7 & 2256 (9.8\%) \\
        \quad $\lambda=2.0 \times 10^{-5}$ & 93.1 & 239 (13.9\%) & 92.3 & 983 (12.7\%) & 64.1 & 5265 (12.3\%) & 83.2 & 2162 (13.0\%) \\
        \quad $\lambda=5.0 \times 10^{-5}$ & 93.1 & 241 (13.4\%) & 91.7 & 916 (18.7\%) & 60.7 & 4460 (25.7\%) & 81.8 & 1872 (19.3\%) \\
        \quad $\lambda=1.0 \times 10^{-4}$ & 92.7 & 223 (19.8\%) & 89.4 & 728 (35.4\%) & 50.5 & 2669 (55.5\%) & 77.6 & 1206 (36.9\%) \\
        \quad $\lambda=2.0 \times 10^{-4}$ & 91.7 & 214 (23.0\%) & 83.3 & 526 (53.3\%) & 29.4 & 1067 (82.2\%) & 68.2 & 602 (52.9\%) \\
        \addlinespace[0.2ex]
        \bottomrule
        \end{tabular}
    }
    \caption{Full results for GPT-OSS-20B with forced answer.}
    \label{tab:full_gpt_oss_20b_forced_answer}
\end{table*}

\begin{table*}[htbp]
    \centering
    \resizebox{\textwidth}{!}{%
        \begin{tabular}{lcccccccc}
        \toprule
        \multicolumn{9}{c}{\textbf{DRPO-7B}} \\
        \multirow{2}{*}{Method / Param} & \multicolumn{2}{c}{GSM8K} & \multicolumn{2}{c}{MATH-500} & \multicolumn{2}{c}{AIME} & \multicolumn{2}{c}{\textit{AVERAGE}} \\
        \cmidrule(lr){2-3} \cmidrule(lr){4-5} \cmidrule(lr){6-7} \cmidrule(lr){8-9}
         & Acc & Length (Red.) & Acc & Length (Red.) & Acc & Length (Red.) & Acc & Length (Red.) \\
        \midrule
        \textit{Baseline} \\
        \quad N/A & 91.5 & 398 (0.0\%) & 85.6 & 1096 (0.0\%) & 44.4 & 4359 (0.0\%) & 73.8 & 1951 (0.0\%) \\
        \addlinespace[0.2ex]
        \textit{Classification Probe} \\
        \quad $\gamma=0.70$ & 87.2 & 306 (23.0\%) & 83.4 & 937 (14.5\%) & 42.9 & 4268 (2.1\%) & 71.2 & 1837 (13.2\%) \\
        \quad $\gamma=0.75$ & 89.3 & 322 (19.0\%) & 84.9 & 974 (11.1\%) & 43.7 & 4300 (1.3\%) & 72.7 & 1866 (10.5\%) \\
        \quad $\gamma=0.80$ & 90.5 & 338 (15.0\%) & 85.8 & 1009 (8.0\%) & 44.2 & 4325 (0.8\%) & 73.5 & 1891 (7.9\%) \\
        \quad $\gamma=0.85$ & 91.2 & 357 (10.3\%) & 85.8 & 1044 (4.8\%) & 44.4 & 4354 (0.1\%) & 73.8 & 1918 (5.1\%) \\
        \quad $\gamma=0.90$ & 91.4 & 380 (4.4\%) & 85.6 & 1076 (1.9\%) & 44.4 & 4356 (0.1\%) & 73.8 & 1937 (2.1\%) \\
        \addlinespace[0.2ex]
        \textit{Answer Convergence} \\
        \quad $\gamma=0.70$ & 84.2 & 288 (27.7\%) & 83.1 & 948 (13.5\%) & 40.8 & 4021 (7.8\%) & 69.4 & 1752 (16.3\%) \\
        \quad $\gamma=0.75$ & 86.9 & 301 (24.4\%) & 83.8 & 977 (10.9\%) & 41.7 & 4127 (5.3\%) & 70.8 & 1802 (13.6\%) \\
        \quad $\gamma=0.80$ & 88.9 & 315 (20.7\%) & 84.6 & 1004 (8.5\%) & 42.8 & 4205 (3.5\%) & 72.1 & 1841 (10.9\%) \\
        \quad $\gamma=0.85$ & 90.3 & 331 (16.7\%) & 85.0 & 1034 (5.7\%) & 43.8 & 4261 (2.2\%) & 73.0 & 1876 (8.2\%) \\
        \quad $\gamma=0.90$ & 91.1 & 355 (10.7\%) & 85.3 & 1064 (3.0\%) & 44.2 & 4307 (1.2\%) & 73.5 & 1909 (4.9\%) \\
        \addlinespace[0.2ex]
        \textit{FlashThink} \\
        \quad $\gamma=0.70$ & 88.1 & 288 (27.7\%) & 81.6 & 834 (23.9\%) & 41.8 & 4005 (8.1\%) & 70.5 & 1709 (19.9\%) \\
        \quad $\gamma=0.75$ & 89.3 & 296 (25.7\%) & 83.1 & 865 (21.1\%) & 42.1 & 4052 (7.0\%) & 71.5 & 1738 (17.9\%) \\
        \quad $\gamma=0.80$ & 90.2 & 303 (23.8\%) & 84.0 & 897 (18.2\%) & 42.5 & 4120 (5.5\%) & 72.2 & 1773 (15.8\%) \\
        \quad $\gamma=0.85$ & 90.7 & 310 (22.0\%) & 84.7 & 924 (15.7\%) & 42.9 & 4162 (4.5\%) & 72.8 & 1799 (14.1\%) \\
        \quad $\gamma=0.90$ & 91.0 & 318 (20.1\%) & 84.9 & 958 (12.6\%) & 43.5 & 4221 (3.2\%) & 73.2 & 1832 (12.0\%) \\
        \addlinespace[0.2ex]
        \textit{OS-Pruner} \\
        \quad $\lambda=0$ & 91.2 & 331 (16.8\%) & 87.1 & 956 (12.8\%) & 43.1 & 3870 (11.2\%) & 73.8 & 1719 (13.6\%) \\
        \quad $\lambda=1.0 \times 10^{-5}$ & 91.2 & 330 (16.9\%) & 87.1 & 952 (13.1\%) & 42.9 & 3767 (13.6\%) & 73.7 & 1683 (14.5\%) \\
        \quad $\lambda=2.0 \times 10^{-5}$ & 91.2 & 330 (17.1\%) & 87.0 & 943 (14.0\%) & 42.4 & 3575 (18.0\%) & 73.6 & 1616 (16.4\%) \\
        \quad $\lambda=5.0 \times 10^{-5}$ & 91.2 & 326 (18.0\%) & 86.7 & 891 (18.7\%) & 39.8 & 2819 (35.3\%) & 72.5 & 1346 (24.0\%) \\
        \quad $\lambda=1.0 \times 10^{-4}$ & 91.0 & 322 (19.1\%) & 85.7 & 807 (26.4\%) & 36.9 & 2006 (54.0\%) & 71.2 & 1045 (33.2\%) \\
        \quad $\lambda=2.0 \times 10^{-4}$ & 90.7 & 314 (21.0\%) & 83.0 & 677 (38.3\%) & 30.5 & 1247 (71.4\%) & 68.1 & 746 (43.6\%) \\
        \addlinespace[0.2ex]
        \bottomrule
        \end{tabular}
    }
    \caption{Full results for DRPO-7B with forced answer.}
    \label{tab:full_drpo_7b_forced_answer}
\end{table*}

\begin{table*}[htbp]
    \centering
    \resizebox{\textwidth}{!}{%
        \begin{tabular}{lcccccccc}
        \toprule
        \multicolumn{9}{c}{\textbf{DeepSeek-R1-Distill-Qwen-7B}} \\
        \multirow{2}{*}{Method / Param} & \multicolumn{2}{c}{GSM8K} & \multicolumn{2}{c}{MATH-500} & \multicolumn{2}{c}{AIME} & \multicolumn{2}{c}{\textit{AVERAGE}} \\
        \cmidrule(lr){2-3} \cmidrule(lr){4-5} \cmidrule(lr){6-7} \cmidrule(lr){8-9}
         & Acc & Len (CoT + Ans) (Red.) & Acc & Len (CoT + Ans) (Red.) & Acc & Len (CoT + Ans) (Red.) & Acc & Len (CoT + Ans) (Red.) \\
        \midrule
        \textit{Baseline} \\
        \quad N/A & 90.6 & 1443 (1220+223) (0.0\%) & 91.8 & 3759 (3378+381) (0.0\%) & 46.4 & 10240 (9874+366) (0.0\%) & 76.3 & 5147 (0.0\%) \\
        \addlinespace[0.2ex]
        \textit{Classification Probe} \\
        \quad $\gamma=0.70$ & 85.2 & 841 (628+213) (41.7\%) & 88.5 & 2741 (2233+508) (27.1\%) & 46.1 & 9734 (8355+1379) (4.9\%) & 73.3 & 4439 (24.6\%) \\
        \quad $\gamma=0.75$ & 85.3 & 923 (723+201) (36.0\%) & 89.1 & 2945 (2508+437) (21.7\%) & 46.1 & 9853 (9018+836) (3.8\%) & 73.5 & 4574 (20.5\%) \\
        \quad $\gamma=0.80$ & 86.0 & 1047 (847+200) (27.5\%) & 89.9 & 3169 (2776+393) (15.7\%) & 45.9 & 10018 (9417+600) (2.2\%) & 73.9 & 4745 (15.1\%) \\
        \quad $\gamma=0.85$ & 87.1 & 1217 (1008+209) (15.7\%) & 90.4 & 3435 (3066+369) (8.6\%) & 46.4 & 10133 (9718+415) (1.0\%) & 74.6 & 4928 (8.4\%) \\
        \quad $\gamma=0.90$ & 89.4 & 1374 (1156+218) (4.8\%) & 91.2 & 3646 (3272+374) (3.0\%) & 46.0 & 10237 (9845+392) (0.0\%) & 75.5 & 5086 (2.6\%) \\
        \addlinespace[0.2ex]
        \textit{Answer Convergence} \\
        \quad $\gamma=0.70$ & 84.7 & 795 (596+199) (44.9\%) & 89.5 & 2704 (2341+362) (28.1\%) & 45.8 & 9377 (8838+539) (8.4\%) & 73.4 & 4292 (27.1\%) \\
        \quad $\gamma=0.75$ & 84.9 & 841 (646+194) (41.7\%) & 89.7 & 2782 (2446+336) (26.0\%) & 46.0 & 9436 (9024+412) (7.9\%) & 73.5 & 4353 (25.2\%) \\
        \quad $\gamma=0.80$ & 84.7 & 895 (707+188) (38.0\%) & 89.6 & 2872 (2559+314) (23.6\%) & 46.5 & 9529 (9184+345) (6.9\%) & 73.6 & 4432 (22.8\%) \\
        \quad $\gamma=0.85$ & 85.1 & 969 (785+184) (32.9\%) & 89.8 & 2976 (2676+300) (20.8\%) & 46.1 & 9635 (9351+284) (5.9\%) & 73.7 & 4527 (19.9\%) \\
        \quad $\gamma=0.90$ & 85.2 & 1077 (892+185) (25.4\%) & 90.1 & 3107 (2817+290) (17.3\%) & 46.1 & 9742 (9494+248) (4.9\%) & 73.8 & 4642 (15.9\%) \\
        \addlinespace[0.2ex]
        \textit{FlashThink} \\
        \quad $\gamma=0.70$ & 84.3 & 640 (408+232) (55.7\%) & 88.2 & 2344 (1640+704) (37.6\%) & 46.2 & 9050 (7226+1824) (11.6\%) & 72.9 & 4011 (35.0\%) \\
        \quad $\gamma=0.75$ & 84.9 & 655 (433+222) (54.6\%) & 88.7 & 2377 (1759+619) (36.8\%) & 45.8 & 9081 (7643+1438) (11.3\%) & 73.1 & 4038 (34.2\%) \\
        \quad $\gamma=0.80$ & 84.8 & 669 (459+211) (53.6\%) & 88.8 & 2433 (1886+547) (35.3\%) & 45.7 & 9054 (8004+1050) (11.6\%) & 73.1 & 4052 (33.5\%) \\
        \quad $\gamma=0.85$ & 85.0 & 686 (489+197) (52.5\%) & 88.6 & 2506 (2021+486) (33.3\%) & 45.8 & 9039 (8343+696) (11.7\%) & 73.2 & 4077 (32.5\%) \\
        \quad $\gamma=0.90$ & 84.9 & 722 (527+195) (49.9\%) & 89.0 & 2590 (2193+397) (31.1\%) & 44.8 & 9116 (8578+538) (11.0\%) & 72.9 & 4143 (30.7\%) \\
        \addlinespace[0.2ex]
        \textit{OS-Pruner} \\
        \quad $\lambda=0$ & 84.9 & 834 (658+177) (42.2\%) & 88.5 & 2360 (2127+234) (37.2\%) & 46.1 & 9359 (9225+135) (8.6\%) & 73.2 & 4185 (29.3\%) \\
        \quad $\lambda=1.0 \times 10^{-6}$ & 84.7 & 818 (642+176) (43.3\%) & 88.6 & 2310 (2073+237) (38.5\%) & 46.1 & 9274 (9148+126) (9.4\%) & 73.2 & 4134 (30.4\%) \\
        \quad $\lambda=2.0 \times 10^{-6}$ & 84.7 & 794 (618+176) (45.0\%) & 88.9 & 2268 (2025+243) (39.7\%) & 46.2 & 9200 (9058+141) (10.2\%) & 73.3 & 4087 (31.6\%) \\
        \quad $\lambda=5.0 \times 10^{-6}$ & 84.6 & 771 (593+178) (46.5\%) & 88.5 & 2246 (1980+266) (40.2\%) & 45.9 & 9035 (8849+186) (11.8\%) & 73.0 & 4017 (32.9\%) \\
        \quad $\lambda=1.0 \times 10^{-5}$ & 84.6 & 748 (566+182) (48.2\%) & 88.6 & 2195 (1893+302) (41.6\%) & 46.1 & 8905 (8639+266) (13.0\%) & 73.1 & 3949 (34.3\%) \\
        \quad $\lambda=2.0 \times 10^{-5}$ & 85.0 & 720 (535+185) (50.1\%) & 88.4 & 2155 (1790+364) (42.7\%) & 45.3 & 8839 (8109+730) (13.7\%) & 72.9 & 3905 (35.5\%) \\
        \quad $\lambda=5.0 \times 10^{-5}$ & 85.0 & 693 (500+194) (51.9\%) & 88.3 & 2119 (1617+502) (43.6\%) & 45.2 & 8765 (6553+2211) (14.4\%) & 72.8 & 3859 (36.7\%) \\
        \quad $\lambda=1.0 \times 10^{-4}$ & 85.0 & 686 (482+204) (52.5\%) & 88.3 & 2160 (1347+813) (42.5\%) & 44.2 & 8822 (3601+5220) (13.9\%) & 72.5 & 3889 (36.3\%) \\
        \addlinespace[0.2ex]
        \bottomrule
        \end{tabular}
    }
    \caption{Full results for DeepSeek-R1-Distill-Qwen-7B with free thinking.}
    \label{tab:full_deepseek_r1_distill_qwen_7b_free_thinking}
\end{table*}

\begin{table*}[htbp]
    \centering
    \resizebox{\textwidth}{!}{%
        \begin{tabular}{lcccccccc}
        \toprule
        \multicolumn{9}{c}{\textbf{GPT-OSS-20B}} \\
        \multirow{2}{*}{Method / Param} & \multicolumn{2}{c}{GSM8K} & \multicolumn{2}{c}{MATH-500} & \multicolumn{2}{c}{AIME} & \multicolumn{2}{c}{\textit{AVERAGE}} \\
        \cmidrule(lr){2-3} \cmidrule(lr){4-5} \cmidrule(lr){6-7} \cmidrule(lr){8-9}
         & Acc & Len (CoT + Ans) (Red.) & Acc & Len (CoT + Ans) (Red.) & Acc & Len (CoT + Ans) (Red.) & Acc & Len (CoT + Ans) (Red.) \\
        \midrule
        \textit{Baseline} \\
        \quad N/A & 93.5 & 284 (279+5) (0.0\%) & 95.1 & 1134 (1127+7) (0.0\%) & 65.4 & 6008 (6002+6) (0.0\%) & 84.7 & 2475 (0.0\%) \\
        \addlinespace[0.2ex]
        \textit{Classification Probe} \\
        \quad $\gamma=0.70$ & 90.7 & 224 (218+5) (21.2\%) & 89.8 & 905 (891+14) (20.2\%) & 56.9 & 5442 (5377+65) (9.4\%) & 79.1 & 2190 (16.9\%) \\
        \quad $\gamma=0.75$ & 91.9 & 233 (228+5) (17.8\%) & 92.0 & 961 (950+10) (15.3\%) & 62.3 & 5644 (5610+34) (6.1\%) & 82.0 & 2279 (13.1\%) \\
        \quad $\gamma=0.80$ & 92.7 & 243 (238+5) (14.3\%) & 93.8 & 1012 (1003+9) (10.8\%) & 63.7 & 5776 (5754+22) (3.9\%) & 83.4 & 2344 (9.6\%) \\
        \quad $\gamma=0.85$ & 93.1 & 255 (250+5) (10.1\%) & 94.5 & 1058 (1050+8) (6.7\%) & 65.0 & 5892 (5877+15) (1.9\%) & 84.2 & 2402 (6.2\%) \\
        \quad $\gamma=0.90$ & 93.4 & 266 (261+5) (6.2\%) & 94.9 & 1097 (1091+7) (3.2\%) & 65.6 & 5967 (5957+10) (0.7\%) & 84.6 & 2444 (3.4\%) \\
        \addlinespace[0.2ex]
        \textit{Answer Convergence} \\
        \quad $\gamma=0.70$ & 90.9 & 214 (209+5) (24.6\%) & 93.3 & 972 (958+14) (14.3\%) & 56.4 & 5182 (5104+78) (13.7\%) & 80.2 & 2123 (17.5\%) \\
        \quad $\gamma=0.75$ & 91.7 & 221 (215+5) (22.3\%) & 94.0 & 991 (980+11) (12.6\%) & 59.0 & 5385 (5327+58) (10.4\%) & 81.5 & 2199 (15.1\%) \\
        \quad $\gamma=0.80$ & 92.3 & 228 (223+5) (19.6\%) & 94.4 & 1018 (1008+10) (10.3\%) & 61.3 & 5553 (5505+47) (7.6\%) & 82.7 & 2266 (12.5\%) \\
        \quad $\gamma=0.85$ & 92.8 & 236 (231+5) (16.8\%) & 94.7 & 1043 (1034+9) (8.0\%) & 63.1 & 5643 (5615+27) (6.1\%) & 83.6 & 2307 (10.3\%) \\
        \quad $\gamma=0.90$ & 93.1 & 245 (240+5) (13.7\%) & 95.0 & 1071 (1061+10) (5.6\%) & 64.8 & 5743 (5719+25) (4.4\%) & 84.3 & 2353 (7.9\%) \\
        \addlinespace[0.2ex]
        \textit{FlashThink} \\
        \quad $\gamma=0.70$ & 91.6 & 195 (190+5) (31.3\%) & 91.4 & 823 (810+13) (27.4\%) & 58.0 & 5145 (5103+42) (14.4\%) & 80.3 & 2054 (24.4\%) \\
        \quad $\gamma=0.75$ & 92.4 & 199 (194+5) (29.9\%) & 92.5 & 843 (829+14) (25.7\%) & 59.3 & 5201 (5176+25) (13.4\%) & 81.4 & 2081 (23.0\%) \\
        \quad $\gamma=0.80$ & 92.8 & 202 (197+5) (28.7\%) & 93.1 & 867 (856+11) (23.6\%) & 61.4 & 5277 (5267+9) (12.2\%) & 82.4 & 2115 (21.5\%) \\
        \quad $\gamma=0.85$ & 93.1 & 207 (202+5) (27.0\%) & 94.2 & 890 (880+10) (21.5\%) & 61.8 & 5388 (5377+11) (10.3\%) & 83.0 & 2162 (19.6\%) \\
        \quad $\gamma=0.90$ & 93.2 & 213 (208+5) (24.9\%) & 94.6 & 920 (912+8) (18.9\%) & 62.9 & 5437 (5422+15) (9.5\%) & 83.6 & 2190 (17.7\%) \\
        \addlinespace[0.2ex]
        \textit{OS-Pruner} \\
        \quad $\lambda=0$ & 93.4 & 254 (248+5) (10.7\%) & 95.1 & 1045 (1038+7) (7.9\%) & 64.9 & 5648 (5642+6) (6.0\%) & 84.5 & 2316 (8.2\%) \\
        \quad $\lambda=1.0 \times 10^{-5}$ & 93.4 & 254 (249+5) (10.5\%) & 95.2 & 1025 (1017+8) (9.6\%) & 65.1 & 5552 (5545+6) (7.6\%) & 84.6 & 2277 (9.2\%) \\
        \quad $\lambda=2.0 \times 10^{-5}$ & 93.4 & 247 (242+5) (13.0\%) & 94.9 & 995 (987+8) (12.3\%) & 63.6 & 5334 (5323+11) (11.2\%) & 84.0 & 2192 (12.2\%) \\
        \quad $\lambda=5.0 \times 10^{-5}$ & 93.3 & 249 (243+5) (12.4\%) & 94.5 & 916 (908+9) (19.2\%) & 60.4 & 4401 (4391+10) (26.7\%) & 82.7 & 1855 (19.5\%) \\
        \quad $\lambda=1.0 \times 10^{-4}$ & 93.1 & 233 (228+5) (18.0\%) & 92.5 & 744 (733+12) (34.4\%) & 51.0 & 2885 (2759+126) (52.0\%) & 78.9 & 1287 (34.8\%) \\
        \quad $\lambda=2.0 \times 10^{-4}$ & 92.6 & 226 (221+5) (20.4\%) & 86.6 & 553 (534+18) (51.3\%) & 31.9 & 1329 (1148+182) (77.9\%) & 70.3 & 703 (49.8\%) \\
        \addlinespace[0.2ex]
        \bottomrule
        \end{tabular}
    }
    \caption{Full results for GPT-OSS-20B with free thinking.}
    \label{tab:full_gpt_oss_20b_free_thinking}
\end{table*}

\begin{table*}[htbp]
    \centering
    \resizebox{\textwidth}{!}{%
        \begin{tabular}{lcccccccc}
        \toprule
        \multicolumn{9}{c}{\textbf{DRPO-7B}} \\
        \multirow{2}{*}{Method / Param} & \multicolumn{2}{c}{GSM8K} & \multicolumn{2}{c}{MATH-500} & \multicolumn{2}{c}{AIME} & \multicolumn{2}{c}{\textit{AVERAGE}} \\
        \cmidrule(lr){2-3} \cmidrule(lr){4-5} \cmidrule(lr){6-7} \cmidrule(lr){8-9}
         & Acc & Len (CoT + Ans) (Red.) & Acc & Len (CoT + Ans) (Red.) & Acc & Len (CoT + Ans) (Red.) & Acc & Len (CoT + Ans) (Red.) \\
        \midrule
        \textit{Baseline} \\
        \quad N/A & 91.5 & 430 (399+31) (0.0\%) & 93.2 & 1245 (1097+148) (0.0\%) & 44.3 & 4624 (4360+264) (0.0\%) & 76.3 & 2100 (0.0\%) \\
        \addlinespace[0.2ex]
        \textit{Classification Probe} \\
        \quad $\gamma=0.70$ & 90.6 & 362 (306+56) (15.6\%) & 92.5 & 1086 (934+151) (12.8\%) & 44.1 & 4532 (4255+277) (2.0\%) & 75.7 & 1993 (10.1\%) \\
        \quad $\gamma=0.75$ & 91.0 & 366 (323+43) (14.7\%) & 92.8 & 1108 (974+134) (11.0\%) & 44.3 & 4569 (4301+268) (1.2\%) & 76.0 & 2014 (9.0\%) \\
        \quad $\gamma=0.80$ & 91.2 & 375 (339+36) (12.7\%) & 93.0 & 1141 (1010+131) (8.4\%) & 44.3 & 4594 (4319+275) (0.6\%) & 76.2 & 2037 (7.3\%) \\
        \quad $\gamma=0.85$ & 91.4 & 390 (357+33) (9.2\%) & 93.2 & 1177 (1044+134) (5.4\%) & 44.4 & 4616 (4354+262) (0.2\%) & 76.3 & 2061 (4.9\%) \\
        \quad $\gamma=0.90$ & 91.5 & 411 (380+31) (4.3\%) & 93.2 & 1216 (1075+142) (2.3\%) & 44.4 & 4622 (4357+265) (0.0\%) & 76.4 & 2083 (2.2\%) \\
        \addlinespace[0.2ex]
        \textit{Answer Convergence} \\
        \quad $\gamma=0.70$ & 90.2 & 358 (288+70) (16.7\%) & 92.7 & 1100 (948+152) (11.7\%) & 44.0 & 4306 (3997+309) (6.9\%) & 75.6 & 1921 (11.8\%) \\
        \quad $\gamma=0.75$ & 90.6 & 358 (301+57) (16.6\%) & 92.8 & 1116 (977+139) (10.4\%) & 44.3 & 4358 (4124+234) (5.8\%) & 75.9 & 1944 (10.9\%) \\
        \quad $\gamma=0.80$ & 91.1 & 362 (316+46) (15.7\%) & 92.8 & 1136 (1004+132) (8.7\%) & 44.1 & 4421 (4206+215) (4.4\%) & 76.0 & 1973 (9.6\%) \\
        \quad $\gamma=0.85$ & 91.2 & 370 (332+38) (13.9\%) & 92.9 & 1163 (1034+128) (6.6\%) & 44.1 & 4469 (4263+206) (3.3\%) & 76.1 & 2001 (8.0\%) \\
        \quad $\gamma=0.90$ & 91.4 & 388 (355+33) (9.7\%) & 93.1 & 1198 (1064+134) (3.8\%) & 44.3 & 4521 (4305+217) (2.2\%) & 76.3 & 2036 (5.2\%) \\
        \addlinespace[0.2ex]
        \textit{FlashThink} \\
        \quad $\gamma=0.70$ & 90.8 & 349 (288+62) (18.6\%) & 92.3 & 1058 (834+224) (15.1\%) & 43.9 & 4252 (3996+256) (8.1\%) & 75.7 & 1886 (13.9\%) \\
        \quad $\gamma=0.75$ & 91.0 & 349 (296+53) (18.8\%) & 92.6 & 1060 (864+197) (14.8\%) & 43.9 & 4295 (4069+226) (7.1\%) & 75.8 & 1901 (13.6\%) \\
        \quad $\gamma=0.80$ & 91.0 & 350 (304+46) (18.6\%) & 92.7 & 1069 (895+173) (14.2\%) & 43.6 & 4308 (4120+187) (6.8\%) & 75.8 & 1909 (13.2\%) \\
        \quad $\gamma=0.85$ & 91.3 & 350 (311+39) (18.6\%) & 92.7 & 1076 (924+153) (13.5\%) & 44.0 & 4352 (4163+189) (5.9\%) & 76.0 & 1926 (12.7\%) \\
        \quad $\gamma=0.90$ & 91.3 & 354 (318+35) (17.6\%) & 92.5 & 1097 (955+142) (11.9\%) & 44.2 & 4406 (4219+187) (4.7\%) & 76.0 & 1952 (11.4\%) \\
        \addlinespace[0.2ex]
        \textit{OS-Pruner} \\
        \quad $\lambda=0$ & 91.3 & 361 (332+29) (15.9\%) & 92.8 & 1050 (958+92) (15.7\%) & 43.6 & 4127 (3901+226) (10.8\%) & 75.9 & 1846 (14.1\%) \\
        \quad $\lambda=1.0 \times 10^{-5}$ & 91.3 & 361 (332+30) (15.9\%) & 92.8 & 1045 (954+91) (16.1\%) & 44.2 & 4071 (3803+268) (12.0\%) & 76.1 & 1826 (14.6\%) \\
        \quad $\lambda=2.0 \times 10^{-5}$ & 91.3 & 360 (331+29) (16.3\%) & 92.8 & 1041 (945+96) (16.4\%) & 43.6 & 3999 (3636+363) (13.5\%) & 75.9 & 1800 (15.4\%) \\
        \quad $\lambda=5.0 \times 10^{-5}$ & 91.3 & 357 (327+30) (16.8\%) & 92.6 & 1027 (898+129) (17.5\%) & 42.5 & 3806 (2816+991) (17.7\%) & 75.5 & 1730 (17.3\%) \\
        \quad $\lambda=1.0 \times 10^{-4}$ & 91.3 & 355 (323+32) (17.2\%) & 92.3 & 1032 (810+222) (17.2\%) & 42.6 & 3758 (2041+1717) (18.7\%) & 75.4 & 1715 (17.7\%) \\
        \quad $\lambda=2.0 \times 10^{-4}$ & 91.3 & 354 (315+38) (17.6\%) & 92.3 & 1013 (682+331) (18.6\%) & 42.1 & 3353 (1295+2058) (27.5\%) & 75.2 & 1573 (21.2\%) \\
        \addlinespace[0.2ex]
        \bottomrule
        \end{tabular}
    }
    \caption{Full results for DRPO-7B with free thinking.}
    \label{tab:full_drpo_7b_free_thinking}
\end{table*}

\newpage

\end{document}